\documentclass[letterpaper, 10 pt, conference]{ieeeconf}  

\IEEEoverridecommandlockouts                              

\usepackage{amsmath}    
\usepackage{amsfonts}   
\usepackage{amssymb}    
\usepackage{mathtools}  
\usepackage{bm}         

\usepackage{amsthm}
\usepackage{subcaption}

\usepackage[utf8]{inputenc} 
\usepackage[T1]{fontenc}    
\usepackage{hyperref}       
\usepackage{url}            
\usepackage{booktabs}       
\usepackage{nicefrac}       
\usepackage{microtype}      
\usepackage{xcolor}         

\usepackage{soul}  
\usepackage{todonotes}
\usepackage{tcolorbox}

\usepackage{cite}

\DeclareMathOperator*{\ssat}{ssat}
\DeclareMathOperator*{\blkdiag}{blkdiag}

\renewcommand{\vec}[1]{\boldsymbol{#1}}
\newcommand{\norm}[1]{\left\lVert#1\right\rVert}
\newcommand{\umax}{u_{\text{max}}}

\newcommand{\nR}[1]{\mathbb{R}^{#1}}		

\theoremstyle{plain} 
\newtheorem{assumption}{Assumption}     
\newtheorem{proposition}{Proposition}

\theoremstyle{definition}

\theoremstyle{remark}

\newcommand{\wref}{\vec{w}_{\text{ref}}}
\newcommand{\alloc}{M}
\newcommand{\diff}[1]{\frac{d}{d#1}}

\newcommand{\difftwo}[1]{\frac{d^2}{d#1^2}}

\newcommand{\bak}{\text{bak}}
\newcommand{\recon}{\text{rec}}
\renewcommand{\th}{\text{th}}

\title{\LARGE \bf
Multi-Agent Transportation of Free-Flyers in Microgravity Via Pushing Interaction Under Human-in-the-Loop Control
}

\author{
Gregorio Marchesini$^{*}$,
Nicola De Carli$^{*}$,
Sihyun Cho$^{\dagger}$,
Youngkyoung Kong$^{\dagger}$,\\
Elias Krantz$^{*}$,
Mani Hemanth Dhullipalla$^{*}$,
H. Jin Kim$^{\dagger}$, and 
Dimos V. Dimarogonas$^{*}$
\thanks{$^{*}$
KTH Royal Institute of Technology, Stockholm, Sweden.}%
\thanks{$^{\dagger}$ Seoul National University, Seoul, Korea.}
\thanks{This work was supported in part by the National Research Foundation of Korea
(NRF) grant funded by the Korea government (MSIT) (RS-2024-00436984), and in part by the Wallenberg AI, Autonomous Systems and Software Program (WASP) funded by the Knut and Alice Wallenberg (KAW) Foundation, the Swedish Research Council (VR), and Digital Futures.
}
}
\begin{document}

\maketitle
\thispagestyle{empty}
\pagestyle{empty}

\begin{abstract}
We propose a safety-critical framework for the cooperative transportation of passive targets in microgravity, where a team of chaser robots acts through unilateral pushing contacts to track a human-provided desired twist while ensuring safe target motion. The pushing-only nature of the interaction introduces sparse, configuration-dependent actuation constraints requiring chasers to physically relocate on the target body when the desired pushing allocation changes. To address these challenges, we formulate a delay-aware feedback control architecture leveraging Control Lyapunov Function (CLF) and Control Barrier Function (CBF) constraints within a mixed-integer thrust allocation program to enforce stability and safety of the target, respectively. The proposed framework enables reference tracking while guaranteeing obstacle avoidance with a circular obstacle despite intermittent control authority, providing a foundation for human-supervised cooperative transportation of free-flyers in space environments. The proposed framework is validated through Gazebo simulations.
\end{abstract}
\section{Introduction}
There is increasing interest in space robotic systems for autonomous and remotely operated on-orbit missions, including inspection, servicing, assembly, manipulation, and docking \cite{li2022survey,rodrigues2025dynamics,jang2026demonstration,verhagen2026validation}. Among these,
several recent space robotics applications, ranging from on-orbit construction of large-scale structures to space debris removal \cite{li2022survey,rodrigues2025dynamics}, require bringing the state of a passive or poorly actuated target (e.g., a decommissioned spacecraft or a component of a larger structure) from one state to another along a desired trajectory. Several works have focused on using a single chaser spacecraft equipped with a grasping mechanism to transport the target along such a trajectory. However, monolithic robotic systems often face limitations in adaptability, scalability, and mission-specific flexibility. In contrast, multi-agent systems (MAS) are increasingly viewed as a superior alternative, owing to their inherent robustness to individual failures and their ability to provide higher collective thrust for demanding transportation tasks \cite{banerjee2023resiliency}.

\begin{figure}
    \centering
    \includegraphics[width=0.85\linewidth]{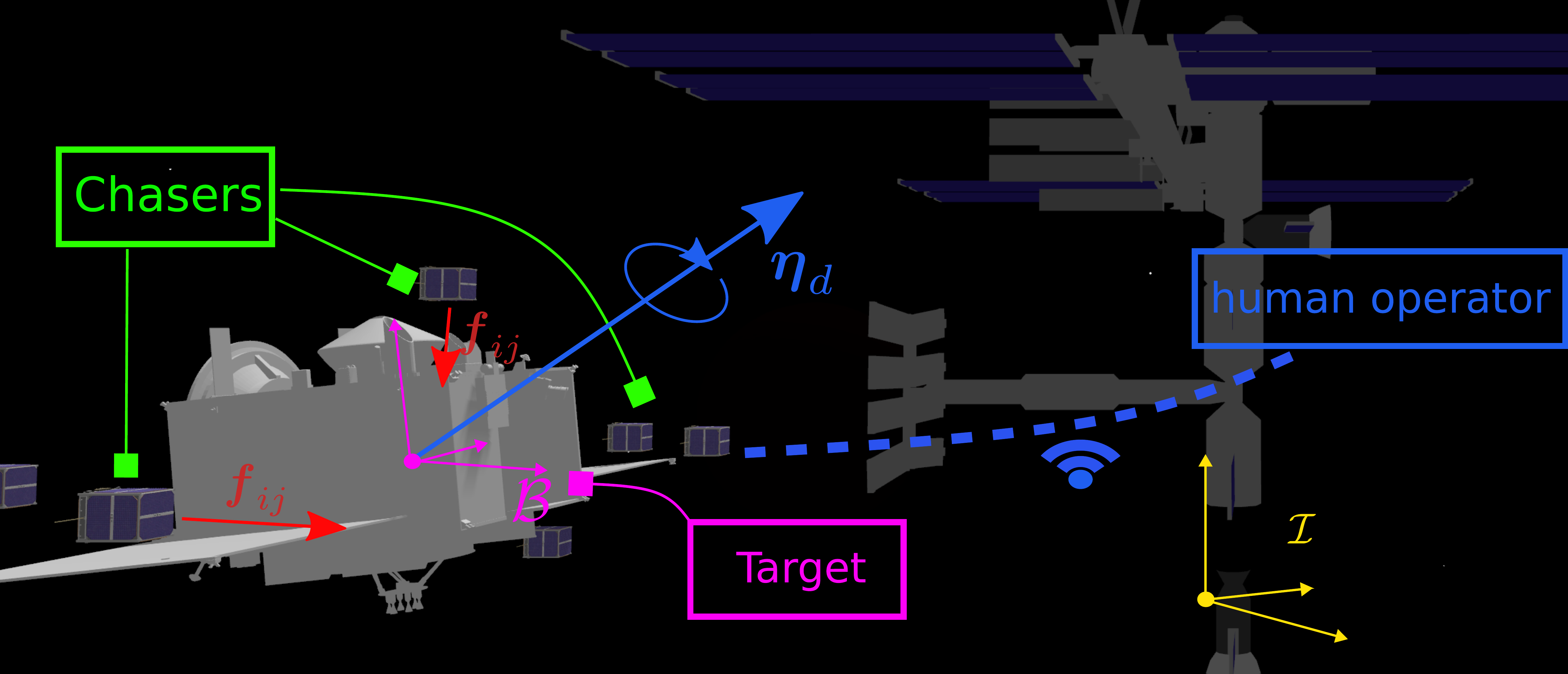}
    \caption{Target-chasers system: a swarm of chasers collaboratively pushes a passive target to follow a human-desired velocity twist. Video : \url{https://youtu.be/W_9vIz7-wOk}}
    \label{fig:high level}
\end{figure}

Several works have demonstrated the applicability of MAS to collaborative target transportation using, for example, tethered  \cite{phodapol2024collaborative}, grasping  \cite{spacecraft_pushing}, impulsive-push  \cite{verhagen2025collaborative}, and soft-pushing interaction \cite{lee2025switchingcontrolunderactuatedmultichannel,gao2024adaptive,gao2023autonomous}. Common to these methods is that the passive target is treated as an unactuated system, whose actuation is virtually enabled by a set of chasers imparting coordinated wrenches on it.

We consider the case in which a passive target is transported in space via pushing interaction by a set of chasers operating in the proximity of a space station (e.g., the International Space Station (ISS)). Pushing interaction is particularly appealing for its simplicity, as it requires no gripping mechanism. We further consider that a human operator provides high-level decision-making in the form of twist commands to navigate the target (e.g., to assemble a given structure), effectively instating a \textit{human-in-the-loop} system (HITL). This enables the application of chasers with limited computational resources, resulting in a cheaper and fault-tolerant system. Within these settings, safety remains a central issue left unaddressed by previous work. Indeed, as chasers periodically relocate themselves on the target to impart the required wrench for tracking twist commands, they intermittently leave the target on a free-flying course. It is during this free-flying transition that safety must be ensured to avoid collisions with other objects in the environment.

Inspired by \cite{lee2025switchingcontrolunderactuatedmultichannel}, this paper proposes a framework for the collaborative transportation of free-floating targets in microgravity. By leveraging pushing-only contact from a swarm of space robots, the proposed approach enables agents to collaboratively transport a passive target while tracking a human-provided velocity reference (see Fig.~\ref{fig:high level}), and ensures collision avoidance with obstacles in the environment. Compared to \cite{lee2025switchingcontrolunderactuatedmultichannel}, we consider references provided by a human operator and introduce an additional safety objective beyond stabilizing the human-given reference.

\subsubsection*{Notation}
Bold letters denote column vectors. The sets $\mathbb{R}$, and $\mathbb{R}_{\geq 0}$ represent real and non-negative real numbers, respectively. Given matrices $A_1,\ldots,A_m$, $\blkdiag(A_1,\ldots,A_m)$ is the block-diagonal matrix with blocks $A_1,\ldots,A_m$. The set $\mathbb{S}^{n} := \{\vec{x} \in \mathbb{R}^{n+1} \mid \| \vec{x} \|=1 \}$, is the unit sphere in $\mathbb{R}^{n+1}$. For a set of $m$ vectors $B = \{\vec{b}_i\}_{i=1}^m \subset \mathbb{R}^{n}$, the \textit{positive cone} of $B$ is $\mathcal{C}(B) = \{ \vec{v} \mid \vec{v} = \sum_{i=1}^{m} \lambda_i \vec{b}_i,\; \lambda_i \geq 0 \}$, and the \textit{strict positive cone} $\mathcal{C}_{\text{st}}(B) = \{ \vec{v} \mid \vec{v} = \sum_{i=1}^m \lambda_i \vec{b}_i,\; \lambda_i > 0\}$. The set $B$ \textit{positively spans} the set $A$ if $A \subseteq \mathcal{C}(B)$. For a continuously differentiable function $g: \mathbb{R}^{n} \rightarrow \mathbb{R}^{m}$, we denote Jacobian as $\diff{\vec{x}}g \in \mathbb{R}^{m\times n}$. When $g$ is scalar function, i.e., $g: \mathbb{R}^{n} \rightarrow \mathbb{R}$, then the Jacobian is the \textit{gradient} of $g$.

\section{Preliminaries}

We consider the problem of controlling a passive free-floating \textbf{target} body with mass leveraging a team of \textbf{chasers}. The chasers can collaboratively exert \textit{push-only} contact on the target (Sec. \ref{sec:actuation model}) to regulate its motion according to a desired velocity command provided by a human operator (see Fig. \ref{fig:pushing object}). The target's dynamics are defined by
\begin{subequations}\label{eq:main dynamics}
\begin{align}
&\diff{t} \vec{p}      = \vec{v}, \; &&\diff{t} \vec{q} = \frac{1}{2}\Theta(\vec{q})\vec{\omega}, \\
&\diff{t} \vec{v}      = \frac{1}{m}R(\vec{q})\vec{f},\; &&\diff{t} \vec{\omega} = J^{-1}\left(-\vec{\omega}\times J\vec{\omega}+\vec{\tau}\right),
\end{align}
\end{subequations}
where $\vec{p},\vec{v} \in \mathbb{R}^{3}$ are the target's inertial position and velocity, $\vec{q} \in \mathbb{S}^3$ is the quaternion representing the orientation of the body frame $\mathcal{B}$ with respect to the inertial frame $\mathcal{I}$ (Fig.~\ref{fig:high level}), and $\vec{\omega} \in \mathbb{R}^3$ is the target's body-frame angular velocity. The vectors $\vec{f},\vec{\tau} \in \mathbb{R}^{3}$ are the body-frame force and torque exerted by the chasers. 
The matrix $J \in \mathbb{R}^{3\times 3}$ is the diagonal body-frame inertia matrix,$m>0$ is the traget mass, and $R(\vec{q}) \in SO(3)$ is the body-to-inertial rotation matrix associated with $\vec{q}$ (see, e.g., \cite[Eq.~2.125]{markley2014correction}). Finally $\Theta(\vec{q})\in \mathbb{R}^{4\times 3}$ is the quaternion dynamics matrix \cite[Eq.~3.20]{markley2014correction}

Letting $\vec{x} = [\vec{p}^\top, \vec{q}^\top]^\top \in \mathbb{R}^{3} \times \mathbb{S}^3$ denote the target \textit{pose}, $\vec{\eta} = [ \vec{v}^\top, \vec{\omega}^\top]^\top \in \mathbb{R}^6$ its \textit{twist}, $\vec{\xi} = [\vec{x}^\top, \vec{\eta}^\top]^\top$ its state, and $\vec{w} = [\vec{f}^\top, \vec{\tau}^\top]^\top \in \mathbb{R}^{6}$ its \textit{wrench}, we express \eqref{eq:main dynamics} in input affine form
\begin{equation}\label{eq:systems dynamics eu}
\begin{aligned}
\diff{t} \vec{\xi} = \begin{bmatrix}
    \diff{t} \vec{x} \\
    \diff{t}  \vec{\eta}
\end{bmatrix} = f_{\xi}(\vec{\xi}) + G_{\xi}(\vec{\xi})\vec{w},
\end{aligned}
\end{equation}
where
\begin{equation}
f_{\xi}(\vec{\xi}) := \begin{bmatrix} G_x(\vec{x})\vec{\eta}\\
f_{\eta}(\vec{x},\vec{\eta})
\end{bmatrix}, \; G_{\xi}(\vec{\xi}) := \begin{bmatrix} \vec{0}_6\\
G_{\eta}(\vec{x}, \vec{\eta})
\end{bmatrix},
\end{equation}
such that  
$G_x(\vec{x}) = \blkdiag\left(\vec{I}_3, \tfrac12\Theta(\vec{q})\right) \in \mathbb{R}^{7\times 6}$, 
$G_\eta(\vec{x},\vec{\eta}) = \blkdiag\left( R(\vec{q}), J^{-1} \right)  \in \mathbb{R}^{6\times 6}$, 
and 
$f_\eta(\vec{x},\vec{\eta}) = \begin{bsmallmatrix}
\vec{0}_3\\
-J^{-1}(\vec{\omega}\times J\vec{\omega})
\end{bsmallmatrix} \in \mathbb{R}^{6}$. Note that $G_x(\vec{x})$ is full column rank for all quaternions $\vec{q}\in\mathbb{S}^3$, and $G_{\eta}(\vec{x}, \vec{\eta})$ is always invertible. We assume the target inertial parameters to be known.

Given an initial state $\vec{\xi}$, we denote by $\phi(t, \vec{\xi})$ the \textit{unforced} solution to \eqref{eq:systems dynamics eu} (i.e., with $\vec{w} =\vec{0}$) at time $t$, by the function 
\begin{equation}\label{eq:unforced dynamics}
\phi(t, \vec{\xi}) := \vec{\xi} + \int_0^{t} f_{\xi}(\phi(s, \vec{\xi})) ds.
\end{equation}

\subsection{Actuation Models}\label{sec:actuation model}
We assume the wrench $\vec{w}$ in \eqref{eq:systems dynamics eu} is applied to the target by a swarm of $n_{a}$ chasers, with the goal of tracking a twist reference $\vec{\eta}_d : [0,\infty) \rightarrow \mathbb{R}^6$ fed by a human operator. Namely, each chaser $j$ interacts with the target over one of the $n_{\th}$ \textit{actuation locations}, $n_{\th}\geq n_{a}$, defined as $l_i = (\vec{r}_i,\vec{d}_{i})$, with position $\vec{r}_i \in \mathbb{R}^3$ and unitary direction vector $\vec{d}_i \in \mathbb{S}^2$ expressed in the target frame $\mathcal{B}$. Chaser $j$ can exert a force $\vec{f}_{ij} = \vec{d}_i u_j$ at the $i$-th target location, where $u_j \in \mathbb{R}_{\geq 0}$, $0\leq u_j \leq \umax$ is the thrust magnitude. We additionally define $\alloc_{ij} \in \{0,1\}$ as a binary variable set to $1$ when chaser $j$ acts on location $i$ and $0$ otherwise. By such interaction, chaser $j$ generates a wrench in the body frame given by
$$
\vec{w}_{ij} = M_{ij}\begin{bmatrix} \vec{f}_{ij} \\ \vec{r}_i \times \vec{f}_{ij} \end{bmatrix}=M_{ij}\begin{bmatrix} \vec{d}_i \\ (\vec{r}_i \times  \vec{d}_i)\end{bmatrix} u_{j} := M_{ij}\vec{b}_i u_{j}.
$$
where $\vec{b}_i = \begin{bmatrix} \vec{d}_i^\top & (\vec{r}_i \times  \vec{d}_i)^\top\end{bmatrix}^\top$. Staking the thrust magnitudes as $\vec{u} = [u_1, \ldots u_{n_a}] \in \mathbb{U}$, with $\mathbb{U}:= [0,u_{max}]^{n_a}$, and letting the \textit{allocation} matrix $\alloc \in \mathbb{M}$,
\begin{equation}\label{eq:allocmat_def}
\mathbb{M}
:=
\left\{
\alloc \in \{0,1\}^{n_{\th}\times n_a}
\;\middle|\;
\sum_{i=1}^{n_{\th}}\alloc_{ij}
\overset{\text{(i)}}{=}1,\;
\sum_{j=1}^{n_a}\alloc_{ij}
\overset{\text{(ii)}}{\leq}1
\right\}.
\end{equation}
then the \textit{admissible} wrench $\vec{w}: \mathbb{M} \times \mathbb{U} \rightarrow \mathbb{R}^{6}$ resulting from all the chasers' interactions is given by:
\begin{equation}\label{eq:pushing wrench}
\vec{w}(\alloc, \vec{u}) = \begin{bmatrix}\vec{f}(M,\vec{u})^\top\\ \vec{\tau}(M,\vec{u})^\top\end{bmatrix} = \sum_{i=1}^{n_{\th}} \vec{b}_i \sum_{j=1}^{n_a}\alloc_{ij} u_{j} = B\alloc\vec{u},
\end{equation}
where $B = [\vec{b}_1, \ldots \vec{b}_{n_{\th}}] \in \mathbb{R}^{6 \times n_{\th}}$. In~\eqref{eq:allocmat_def}, (i) assigns each chaser to exactly one actuation location, while (ii) allows at most one chaser per location. We consider the following assumptions.
\begin{assumption}\label{ass:controllability}
    The set $\{\vec{b}_i\}_{i=1}^{n_{\th}}$ positively spans the wrench space $\mathbb{R}^{6}$, i.e., $\mathcal{C}(\{\vec{b}_i\}_{i=1}^{n_{\th}})=\mathbb{R}^{6}$.
\end{assumption}
\begin{assumption}\label{ass:min chasers}
    The number of chasers is $n_{a} \geq 4$.
\end{assumption}

Assumption~\ref{ass:controllability} is a controllability condition that captures the push-only actuation constraint: each chaser can generate only a nonnegative multiple of the wrench direction $\vec{b}_i$ associated with its contact location, so arbitrary wrenches must be obtained as nonnegative combinations of the available directions. This is closely related to the classical \emph{force-closure} condition in robotic manipulation~\cite[Sec.~12]{lynch2017modern}. Assumption~\ref{ass:min chasers}, on the other hand, sets a minimum number of $4$ agents, as this is the minimum number of agents needed to be able to impart a positive acceleration in every direction in $\mathbb{S}^2$ with a positive cone \cite[Thm.~3.8]{davis1954theory}. Namely, by Assumptions~\ref{ass:controllability} and ~\ref{ass:min chasers}, there existence of at least one fixed allocation $\alloc^{\bak}\in\mathbb{M}$, hereafter termed the \emph{backup allocation}, providing a strictly positive minimum linear acceleration in every direction $\vec{e}\in\mathbb{S}^2$. We define this quantity as
\begin{equation}\label{eq:min-max-acc}
\begin{aligned}
a_{\text{max}}(\alloc^{\bak}) :=  &  \max_{\vec{u} \in \mathbb{U}} \min_{\vec{e}\in \mathbb{S}^{2}} \; \vec{e}^\top \vec{f}(M^{\bak},\vec{u})  \frac{1}{m}, \; \\ 
\end{aligned}
\end{equation}
where $\vec{w}(M^{\bak},\vec{u}) = [\vec{f}(M^{\bak},\vec{u})^\top, \vec{\tau}(M^{\bak},\vec{u})^\top]^\top$ from \eqref{eq:pushing wrench}.  For a fixed actuation geometry $B$ and number of chasers $n_a$, the value of $a_{\max}(\alloc^{\bak})$ depends on the selected backup allocation and can be interpreted as a grasp-quality-like measure of the available translational control authority~\cite[Sec.~12.1.7]{lynch2017modern}. Hence, once $\alloc^{\bak}$ is fixed, the chasers can generate an acceleration of at least $a_{\max}(\alloc^{\bak})>0$ along any translational direction without the need to relocate the target to another configuration.

A feasible $\alloc^{\bak}$ can be constructed geometrically. Let $P = [I_3\;0_3]\in\mathbb{R}^{3\times6}$ denote the projection mapping for a wrench vector $\vec{b}_i$ onto the force components. Select three linearly independent projected wrench directions $\{P\tilde{\vec{b}}_i\}_{i=1}^3$ and a fourth direction satisfying $P\hat{\vec{b}}\in\mathcal{C}_{\mathrm{st}}(\{-P\tilde{\vec{b}}_i\}_{i=1}^3)$~\cite[Thm.~3.8]{davis1954theory}. These four directions positively span $\mathbb{R}^3$, while the remaining $n_a-4$ chasers, if any, may be assigned to arbitrary unoccupied actuation locations.

\section{Stabilization and Safety}

\begin{figure}
    \centering
    \includegraphics[width=0.8\linewidth]{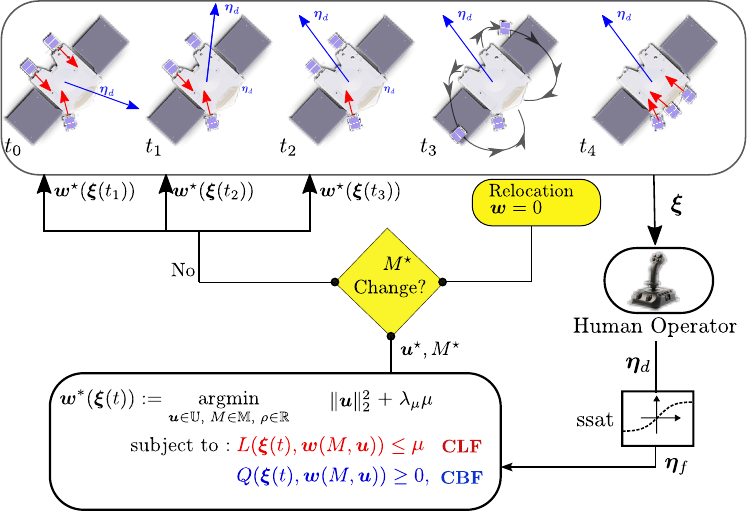}
    \caption{Control architecture of the target-chaser system. The human operator provides a twist command $\vec{\eta}_d$ (filtered through a low-pass filter). The command is fed into a CLF-CBF controller \eqref{eq:cbf-clf base bilinear} from which an optimal allocation and thrust for the chasers is computed. When the optimal allocation changes, chasers relocate, leaving the target unactuated. }
    \label{fig:pushing object}
\end{figure}

We assume that an operator provides a continuous velocity twist signal $\vec{\eta}_d : [0, \infty) \rightarrow \mathbb{R}^6$ to virtually move the target in space. 
The operator may command variations of $\vec{\eta}_d$ that are too rapid for the target to track given its limited actuation authority; moreover, the derivative $\dot{\vec{\eta}}_d$ may not be available. We therefore introduce a filtered reference twist $\vec{\eta}_f$ with bounded rate of change, governed by
\begin{equation}\label{eq:filtered_vel}
    \diff{t} \vec{\eta}_f = -\ssat\left(k_f(\vec{\eta}_f - \vec{\eta}_d), 
    \dot{\eta}_{\max}
    \right)
\end{equation}
where $k_{f} >0$ is the filter gain, $\dot{\eta}_{\max}>0$ is an upper bound on the norm of the filtered reference derivative, and $\ssat:\nR{n}\times \nR{}\to \nR{n}$ denotes a smooth saturation operator 
$\ssat(\vec{\eta}, \eta_{\max}) \coloneqq \eta_{\max} \tanh\left(\frac{\norm{\boldsymbol{\eta}}}{\eta_{\max}}
    \right)
    \frac{\boldsymbol{\eta}}{\norm{\boldsymbol{\eta}}}.$

The objective of the chasers is to apply a wrench on the target to track the reference twist $\vec{\eta}_f$ provided by the human operator while avoiding a spherical obstacle, representing, for example, the space station from which the mission is operated. A key challenge is that changes in the chaser allocation require a finite relocation time during which the target may experience reduced or no control authority. We formalize this as follows.
\begin{assumption}
    There exists a fixed finite $\delta>0$ such that any relocation between two distinct allocations $M_1,M_2\in\mathbb{M}$ requires at most a time $\delta$.
\end{assumption}  
During the interval $\delta$, the target is therefore uncontrolled (i.e., $\vec{w}=0$) and may enter an unsafe configuration leading to an unavoidable collision. In the next section, assuming no reconfiguration time is required, we define a control approach, based on the notion of  Control Lyapunov Functions (CLF) and Control Barrier Functions (CBF) \cite{romdlony2016stabilization}, that stabilizes the target toward the reference while maintaining safety, following the architecture in Fig.~\ref{fig:pushing object}. In Section~\ref{sec:enhanced_safey}, we then show how this safety definition can be extended to account for the relocation time $\delta$, and how a hybrid control architecture that modulates the safety constraints based on the system's current state reduces conservatism.

\subsection{Stabilization}
To track the filtered reference twist $\vec{\eta}_f$, 
we consider the smooth Lyapunov function
\begin{equation}
V(\vec{\xi}) := \frac{1}{2}\| \vec{e}_{\eta} \|^2, \quad  \vec{e}_{\eta}:= \vec{\eta} - \vec{\eta}_f,
\end{equation}
with time derivative (omitting function arguments for brevity)
\begin{equation}\label{eq:lyapunov derivative}
 L(\vec{\xi}, \vec{w}):= \diff{t}{V} = \diff{\vec{\xi}}V^T \diff{t}\vec{\xi} = \vec{e}_{\eta}^\top(f_{\eta} + G_{\eta}\vec{w} - \diff{t} \vec{\eta}_f).
\end{equation}
For a given wrench signal $\vec{w}(t)$ such that $L(\vec{\xi}(t), \vec{w}(t)) \leq - \epsilon$, $\epsilon >0$, then $\vec{e}_{\vec{\eta}}(t)$ is asymptotically stabilized to the origin \cite[Thm. 4.2]{khalil2002nonlinear}. In particular, replacing a reference wrench input $\wref(\vec{\xi}) = G_{\eta}^{-1}(-f_{\eta} + \diff{t} \vec{\eta}_f - K \vec{e}_{\eta})$
with $K >0$, we obtain $L(\vec{\xi}, \wref) = \diff{t}{V} = -K\|\vec{e}_{\eta}\|^2 \leq 0$, which yields exponential convergence of the tracking error,
\(
\|\vec{e}_{\eta}(t)\|
=
\|\vec{e}_{\eta}(0)\|e^{-Kt}
\). Adding and subtracting $\wref$ to $\vec{w}$ in \eqref{eq:lyapunov derivative}, we can thus equivalently write $L(\vec{\xi}, \vec{w})$ as
\begin{equation}
L(\vec{\xi}, \vec{w})= \vec{q}_{V}(\vec{\xi})^\top \vec{w} - m_V(\vec{\xi}),
\end{equation}
where
\begin{equation}
\vec{q}_{V}(\vec{\xi}) := G_{\eta}^\top \vec{e}_{\eta}, \; m_V(\vec{\xi}) := -K \norm{\vec{e}_{\eta}}^2
\end{equation}
\subsection{Safety}
\begin{figure}[ht]
    \centering
    \includegraphics[width=0.6\linewidth]{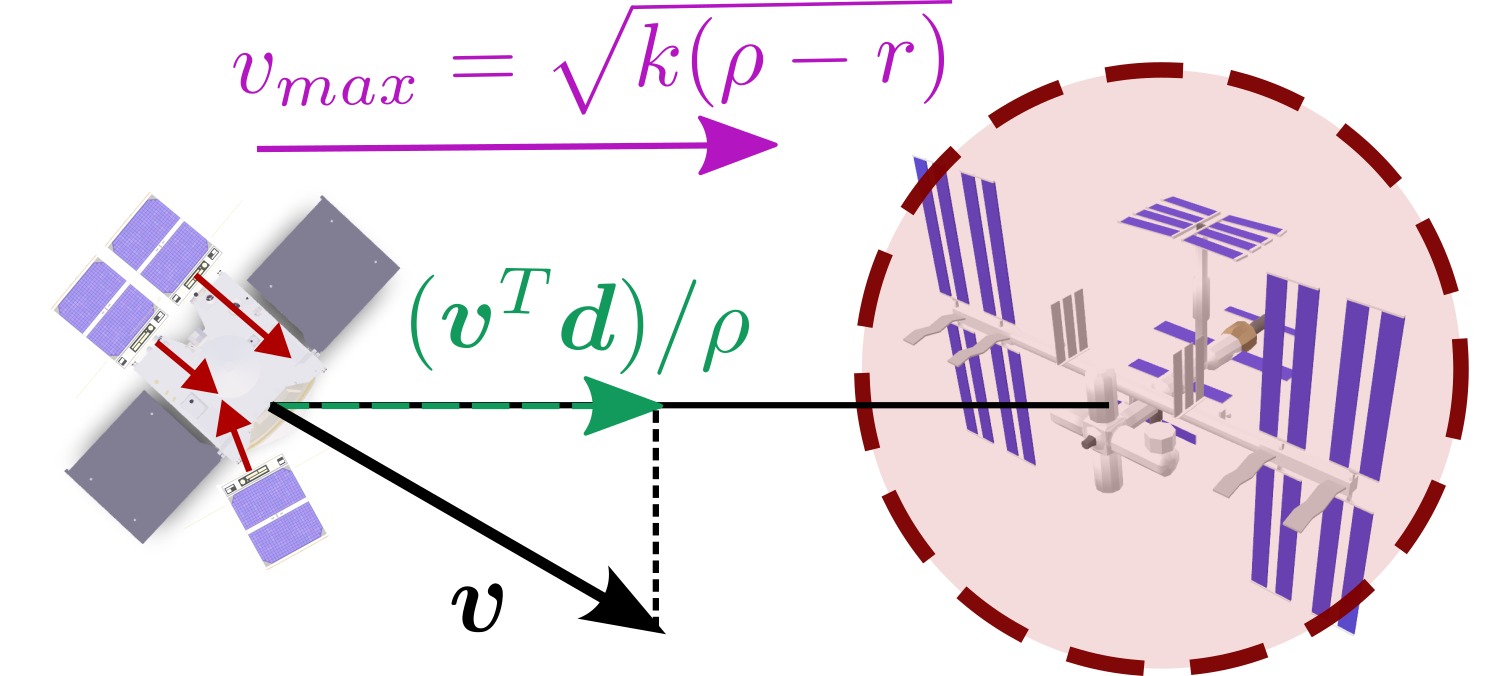}
    \caption{Barrier representation.}
    \label{fig:collision}
\end{figure}

As a safety objective, we aim to avoid a spherical obstacle in the environment with position $\vec{c} \in \mathbb{R}^3$ and radius $r>0$. Namely, let the scalar function
\begin{equation}
\begin{aligned}
h(\vec{\xi}) &= 
\rho(\vec{\xi})-r,  
\end{aligned}
\end{equation}
where $\rho(\vec{\xi}):=\|\vec{d}\|$, $\vec{d}:=\vec{p}-\vec{c}$. The target safe set is given by $\mathcal{H} := \{ \vec{\xi} \mid  h(\vec{\xi}) \geq 0 \}$ (i.e., all states for which the position $\vec{p}$ is outside the obstacle). 

To guarantee safety, we consider the control barrier function
\begin{equation}\label{eq:barrier functions}
b(\vec{\xi}) = \underbrace{\frac{\vec{v}^\top\vec{d}}{\rho}}_{{d h}/{dt}} + \sqrt{k h(\vec{\xi}) + \varepsilon} - \sqrt{\varepsilon},\\
\end{equation}
where $k := 2 a_{\text{max}}(\alloc^{\bak})$, with $a_{\text{max}}(\alloc^{\bak})>0$ as per \eqref{eq:min-max-acc}, and $\varepsilon>0$ a regularization constant. For $\rho\geq r$ (or equivalently $h\geq 0$), the regularization makes $b$ continuously differentiable at $\rho=r$ and provides a conservative approximation, since
\(
\sqrt{kh+\varepsilon}-\sqrt{\varepsilon}
\leq
\sqrt{kh}.
\)
The first term in $b$ is the outward radial velocity, while $\sqrt{k(\alloc^{\bak})h}=\sqrt{2a_{\max}(\alloc^{\bak})h}$ is the largest admissible approach speed from which the target can be brought to rest before reaching the obstacle using the guaranteed maximum outward deceleration~\cite[Sec.~III]{safety_barrier}.
The regularized second term therefore provides a smooth conservative approximation of this braking bound.
Moreover, from \eqref{eq:barrier functions}, the condition $b(\vec{\xi})\geq0$ implies
\[
\diff{t}h
\geq
-\left(
\sqrt{kh+\varepsilon}-\sqrt{\varepsilon}
\right).
\]
Since the right-hand side vanishes at $h=0$, any trajectory starting in $\mathcal{H}$ and satisfying $b(\vec{\xi}(t))\geq0$ remains in $\mathcal{H}$ based on Nagumo's theorem \cite{ames2019control}.
Hence, we define \( \mathcal{B}:=\{\vec{\xi}\in \mathcal{H}\mid b(\vec{\xi})\geq0\}, \) with $\mathcal{B}\subseteq\mathcal{H}$. Assuming $\vec{\xi}(t_0)\in\mathcal{B}$, the condition
\begin{equation}\label{eq:cbf condition}
Q(\vec{\xi}(t),\vec{w}(t))
:=
\diff{\vec{\xi}}b^\top\diff{t}\vec{\xi}(t)
+\kappa b(\vec{\xi}(t))
\geq0,
\end{equation}
for all $t\geq t_0$ and some $\kappa>0$, this guarantees $b(\vec{\xi}(t))\geq0$ for all $t\geq t_0$~\cite{ames2019control}.
Together with the argument above, i.e. $\frac{dh}{dt} \geq 0$ at $h=0$, this ensures $\vec{\xi}(t)\in\mathcal{B}$ for all $t\geq t_0$. Noting that
\begin{equation}
\diff{\vec{\xi}}b^T \diff{t}\vec{\xi} = \frac{\|\vec{v}\|^2}{\rho} - \frac{(\vec{d}^\top\vec{v})^2}{\rho^3} +  \frac{k \cdot \vec{d}^\top\vec{v}}{2\rho \sqrt{kh + \varepsilon}} + \frac{\vec{d}^\top R(\vec{q}) \vec{f}}{\rho m},
\end{equation}
and letting
$$
\begin{aligned}
\vec{q}_{Q}(\vec{\xi}) &:= \begin{bmatrix}\frac{\vec{d}^\top R(\vec{q}) }{\rho m}, 0^T_3\end{bmatrix}^\top,\\
m_{Q}(\vec{\xi})       &:=  -\frac{\|\vec{v}\|^2}{\rho} + \frac{(\vec{d}^\top\vec{v})^2}{\rho^3} -  \frac{k \cdot \vec{d}^\top\vec{v}}{2\rho \sqrt{kh+ \varepsilon}} -  \kappa b(\vec{\xi}),
\end{aligned}
$$ 
the CBF condition can be compactly written as
\begin{equation}
Q(\vec{\xi},\vec{w}) = \vec{q}_{Q}(\vec{\xi})^\top \vec{w} - m_{Q}(\vec{\xi}) \geq 0.
\end{equation}

\subsection{Controller Design}\label{sec:controller design}

Under the actuation model \eqref{eq:pushing wrench}, and neglecting chaser relocation times, the stabilization and safety objectives can be addressed through the following CLF--CBF controller $\vec{w}(\vec{\xi}(t)) = \vec{w}^{\star}(\vec{\xi}(t))$ :
\begin{subequations}\label{eq:cbf-clf base bilinear}
\begin{align}
\vec{w}^{*}(\vec{\xi}(t)) &:= \underset{\vec{u} \in \mathbb{U}, \;  \alloc \in \mathbb{M}, \; \mu \in \mathbb{R}_{\geq0}}{\text{argmin}} \qquad \|\vec{u}\|^2_{2} + \lambda_{\mu}\mu \label{eq:basecost} \\ 
&\qquad\quad\text{subject to :}  \;       L(\vec{\xi}(t),\vec{w}(M,\vec{u})) \leq \mu \label{eq:baseconst:stability}\\
&\qquad\quad\phantom{\text{subject to :}}\; Q(\vec{\xi}(t),\vec{w}(M,\vec{u})) \geq 0,   \label{eq:baseconst:safety}
\end{align}
\end{subequations}
where, $\vec{w}^{*}(\vec{\xi}(t)) = \vec{w}(\alloc^{*}(\vec{\xi}), \vec{u}^{*}(\vec{\xi}))$ should be understood as the optimal wrench computed as per \eqref{eq:pushing wrench}, by the optimal allocation $\alloc^\star(\vec{\xi})$ and thrust $\vec{u}^{*}(\vec{\xi})$, obtained by solving \eqref{eq:cbf-clf base bilinear} at state $\vec{\xi}(t)$. In \eqref{eq:cbf-clf base bilinear} we highlight that: (i) ~\eqref{eq:baseconst:stability} imposes the CLF condition with a non-negative slack variable $\mu$, penalized in~\eqref{eq:basecost} by the weight $\lambda_\mu>0$, 
(ii) constraint \eqref{eq:baseconst:safety} collects the CBF safety constraints, and (iii) the cost function \eqref{eq:basecost} penalizes the control effort and violations of the CLF condition.

At each state $\vec{\xi}$, \eqref{eq:cbf-clf base bilinear} is a mixed integer quadratic program (MIQP), with bilinear constraints since both $Q(\vec{\xi},\vec{w})$ and $L(\vec{\xi},\vec{w})$ are linear in $\vec{w}$, but $\vec{w}(M,\vec{u}) = B\alloc \vec{u}$, which is bilinear in the variables $M$ and $\vec{u}$, as per \eqref{eq:pushing wrench}. McCormick envelopes constraints \cite[Sec. 5.2.1]{belotti2013mixed} are a standard method to replace the bilinear constraints with equivalent linear constraints by introducing auxiliary variables, making \eqref{eq:cbf-clf base bilinear} solvable via branch-and-bound solvers. While such problems are renowned to be computationally hard to solve, we will show in Sec.~\ref{sec:relaxation} how we can efficiently relax the problem to meet real-time computational constraints.

As previously highlighted, the implementation of \eqref{eq:cbf-clf base bilinear} does not consider the relocation time $\delta$. Namely, if for two distinct times $t_2 > t_1$ the optimal allocation obtained by \eqref{eq:cbf-clf base bilinear} changes, i.e., $\alloc^{*}(\vec{\xi}(t_2)) \neq \alloc^{*}(\vec{\xi}(t_1))$, then the chasers must relocate on the target. This then leaves the target uncontrolled for an interval $\delta$ (see Fig~\ref{fig:high level}). Our safety guarantees must therefore be strengthened to ensure collisions are avoided in this interval.
\section{Enhanced Safety Under Allocation Switch}\label{sec:enhanced_safey}
Inspired by \cite{backupcbf,adrian}, we consider a \textit{predictive} safe set $\mathcal{B}^\star \subset \mathcal{B}$ such that:
\begin{equation}
\vec{\xi} \in  \mathcal{B}^\star \; \Rightarrow \;  \phi(\tau, \vec{\xi}) \in  \mathcal{B}, \; \forall \tau \in [0,\delta].
\end{equation}
where $\phi(\tau, \vec{\xi})$ is the unforced solution of the dynamics of the target as per \eqref{eq:unforced dynamics}. We characterize $\mathcal{B}^\star$ as the superlevel set $\mathcal{B}^\star := \{\vec{\xi} \mid b^{\star}(\vec{\xi}) \geq 0\}$ of the predictive CBF
\begin{equation}\label{eq:optimal barrier}
b^{\star}(\vec{\xi}) := \underset{\tau \in [0,\delta]}{\min}\;  b(\phi(\tau, \vec{\xi})).
\end{equation}
The corresponding set of minimizers is $T^\star:= \{\tau^{\star} \mid  b(\phi(\tau^{\star}, \vec{\xi}) = b^{\star}(\vec{\xi}) \}$.
Thus, $b^{\star}(\vec{\xi})$ gives the minimum value attained by the original CBF \eqref{eq:barrier functions} when the target evolves \textit{uncontrolled}, i.e. with $\vec{w}=0$, over the interval $[0,\delta]$. Since $\phi(0,\vec{\xi})=\vec{\xi}$, it follows that
\(b^\star(\vec{\xi})\leq b(\vec{\xi}),\) and therefore $\mathcal{B}^\star := \{ \vec{\xi} \mid b^{\star}(\vec{\xi}) \geq 0 \} \subset \mathcal{B}$. The minimum defined by $b^{\star}(\vec{\xi})$ in \eqref{eq:optimal barrier} can be obtained efficiently, as it is a one-dimensional optimization over the variable $\tau$ (for example, bi-bisection) and the unforced flow $\phi(\tau^{\star}, \vec{\xi})$ of \eqref{eq:main dynamics} admits an explicit analytical solution. We prove shortly in  Proposition \ref{prop:main propsition} that such a minimum is unique.

Similar to $\mathcal{B}$, we can preserve the system state in $\mathcal{B}^{\star}$ by enforcing the constraint 
\begin{equation}\label{eq:new safety constraint}
Q^{\star}(\vec{\xi}(t), \vec{w}(t)) := \diff{\vec{\xi}}b^{\star \top} \diff{t}\vec{\xi}(t) + \kappa b^{\star}(\vec{\xi}(t)) \geq 0,
\end{equation}
for all $t \in [0, \infty)$. However, when considering $b^{\star}(\vec{\xi})$ as a candidate CBF for our system, we need to prove its differentiability. Using sensitivity analysis of the optimization program \eqref{eq:optimal barrier}, we can prove that $b^{\star}(\vec{\xi})$ is Lipschitz continuous and differentiable almost everywhere (i.e., except on a set of measure zero). 
\begin{proposition}\label{prop:main propsition}
Define \( \mathcal{V} := \left\{ \vec{\xi}\mid \|\vec{v}\|=0 \right\}. \)
Then, for every
\(
\vec{\xi}\in
\mathcal{B}_{+}^{\star}
:=
\mathcal{B}^{\star}\setminus\mathcal{V},
\)
the minimizer of~\eqref{eq:optimal barrier} is unique, i.e.,
\(
T^\star(\vec{\xi})
=
\{\tau^\star_{\vec{\xi}}\}
\), 
\(
\tau^\star_{\vec{\xi}}
:=
\tau^\star(\vec{\xi}),
\)
and $b^\star$ is differentiable at $\vec{\xi}$, with gradient 
\begin{equation}\label{eq:derivative bstar}
\diff{\vec{\xi}}b^\star(\vec{\xi})
=
\Phi(\tau^\star_{\vec{\xi}},\vec{\xi})^\top
\diff{\vec{\xi}}b
\left(
\phi(\tau^\star_{\vec{\xi}},\vec{\xi})
\right),
\end{equation}
where
\begin{equation}\label{eq:state transition matrix}
\Phi(\tau,\vec{\xi})
=
I_{13}
+
\int_0^\tau
\diff{\vec{\xi}} f_{\xi}\!\left(\phi(s,\vec{\xi})\right)
\Phi(s,\vec{\xi})\,ds,
\end{equation}
is the state-sensitivity matrix of the unforced flow $\phi(\cdot,\vec{\xi})$.
\end{proposition}
\begin{proof}
    See Appendix \ref{appendix:proof result}.
\end{proof}
Note that the set $\mathcal{V}$ does not contain unsafe states, but is excluded from
$\mathcal{B}_{+}^{\star}$ because the minimizer in~\eqref{eq:optimal barrier}
is nonunique. Indeed, under the unforced dynamics, $\vec{v}=0$ implies
$\vec{p}(\tau)=\vec{p}$ for all $\tau\in[0,\delta]$, and therefore
$b(\phi(\tau,\vec{\xi}))=b(\vec{\xi})$ over the entire prediction horizon.
Hence, $T^\star(\vec{\xi})=[0,\delta]$. This non-differentiability thus does not represent a problem since, when
$\vec{\xi}\in\mathcal{V}\cap\mathcal{B}^{\star}$ the target remains safe
under the unforced dynamics. Proposition \ref{prop:main propsition} therefore allows the safety constraint \eqref{eq:new safety constraint} to be evaluated using the gradient \eqref{eq:derivative bstar} for all
$\vec{\xi}\in\mathcal{B}_{+}^{\star}$. 
\section{Controller architecture}
Following the hybrid-systems formalism in~\cite{lygeros2003dynamical}, we propose a controller architecture that switches among different control modes depending on the state of the system and the required chaser allocation. At a high level, while $\vec{\xi}\in\mathcal{B}^{\star}$, the chasers may optimize their allocation $M$. When operation in $\mathcal{B}^{\star}$ can no longer be maintained, the chasers relocate to the backup allocation $M^{\bak}$, from which safety can always be enforced (Proposition~\ref{prop:last prop}).

To define our controller, we consider a timer state $s \in [0,\delta]$, and the \textit{augmented} state $\vec{z} := (\vec{\xi}, s)$. We thus formally define the hybrid system that defines our controller architecture as:
\begin{equation}\label{eq:hybrid system}
H := (\mathcal{Z}, \mathcal{Q}, \vec{w}(\vec{\xi}, q), \Sigma, G, R, \text{Inv}, \text{Init}).
\end{equation}
The set $\mathcal{Z} := \mathcal{B} \times [0,\delta]$
is the domain of the continuous state $\vec{z}$, and
$\mathcal{Q} := \{q^{\star}, q^{\bak}, q^{\recon}\}$ is a set of discrete states,
or \emph{modes}, respectively termed as \textit{primary}, \textit{backup}, and \textit{relocation} states. The map $\vec{w}(\vec{\xi}, q)$ assigns a feedback controller to
each mode, so that the continuous state evolves as
\begin{equation}\label{eq:continuous dynamics}
\hspace{-0.2cm}\diff{t}\vec{\xi} = f_{\xi}(\vec{\xi}) + G_{\xi}(\vec{\xi}) \, \vec{w}(\vec{\xi}, q),
\; 
\diff{t}s = \begin{cases} 1 & \text{if } q = q^{\recon}, \\ 0 & \text{else.}\end{cases}
\end{equation}
The relation $\Sigma \subseteq \mathcal{Q}^{2}$ collects the transitions between
modes, where $q_1 \rightarrow q_2$ denotes a transition from $q_1$ to $q_2$. The
map $G : \Sigma \rightarrow 2^{\mathcal{Z}}$ assigns to each transition a
\emph{guard}, that is, a subset of the continuous state space within which that discrete
transition may be taken, while $R : \Sigma \times \mathcal{Z} \rightarrow
\mathcal{Z}$ assigns to each transition a \emph{reset map}. The map
$\text{Inv} : \mathcal{Q} \rightarrow 2^{\mathcal{Z}}$ assigns to each mode an
\emph{invariant}, the region in which the system is permitted to remain in that
mode. Finally, $\text{Init}$ is the set of admissible initial conditions. We now
specify each of these elements for the architecture at hand.

\noindent
\textbf{Feedback control modes}:
Two modes correspond to the two CLF-CBF feedback controllers:
\begin{tcolorbox}[colback=gray!20, colframe=gray!20, arc=0mm, top=0mm, bottom=0mm, left=0mm, right=0mm, boxsep=0mm]
\begin{subequations}\label{eq:backup controller}
\begin{align}
&\hspace{-0.2cm}\vec{w}^{\star}(\vec{\xi}, q^{\bak}) := \underset{\vec{u} \in \mathbb{U}, \; \mu \in \mathbb{R}_{\geq0}}{\text{argmin}} \, \|\vec{u}\|^2_{2} + \lambda_{\mu}\mu, \; \text{subject to:}\\
&\hspace{-0.2cm}L(\vec{\xi}, \vec{w}(M^{\bak}, \vec{u})) \leq \mu, \;  Q(\vec{\xi}, \vec{w}(M^{\bak}, \vec{u})) \geq 0, \label{eq:backup controller:constraints}
\end{align}
\end{subequations}
\end{tcolorbox}
and
\begin{tcolorbox}[colback=gray!20, colframe=gray!20, arc=0mm, top=0mm, bottom=0mm, left=0mm, right=0mm, boxsep=0mm]
\begin{subequations}\label{eq:cool controller}
\begin{align}
&\hspace{-0.2cm}\vec{w}^{\star}(\vec{\xi}, q^{\star}) := \hspace{-0.5cm} \underset{\vec{u} \in \mathbb{U}, \alloc \in \mathbb{M}, \; \mu \in \mathbb{R}_{\geq0}}{\text{argmin}} \hspace{-0.1cm} \|\vec{u}\|_{2}^2 + \lambda_{\mu}\mu, \; \text{subject to:}\\
&\hspace{-0.2cm}L(\vec{\xi}, \vec{w}(M, \vec{u})) \leq \mu, \;  Q^{\star}(\vec{\xi}, \vec{w}(M, \vec{u})) \geq 0, \label{eq:primary controller:constraints}
\end{align}
\end{subequations}
\end{tcolorbox}
which we refer to as the \textit{backup controller} and the \textit{primary
controller}, respectively. In the \textit{backup} mode, the configuration is fixed to $M^{\bak}$, while in the primary controller, the configuration $M \in \mathbb{M}$ is optimized and can thus be changed at the expense of triggering a relocation. In the third mode $q^{\recon}$, the \textit{relocation} mode, we let $\vec{w}^{\star}(\vec{\xi}(t), q^{\recon}) := \vec{0}$, with the target following the unforced flow $\phi(\cdot,\vec{\xi})$, while the timer $s$ measures the elapsed relocation time as per \eqref{eq:continuous dynamics}.

\noindent
\textbf{\textit{Invariants}}
The primary controller is only applied on $\mathcal{B}^{\star} \subset \mathcal{B}$,
whereas the backup controller is allowed on all of $\mathcal{B}$. During
relocation, the timer $s$ will also vary from $0$ to $\delta$. We thus have
$\text{Inv}(q^{\star})  = \mathcal{B}^{\star} \times \{0\}, \; \text{Inv}(q^{\bak})   = \mathcal{B} \times \{0\}, \\ \text{Inv}(q^{\recon}) = \mathcal{B} \times [0,\delta].$

\noindent
\textbf{Guards and resets}
Every change of allocation $M$ forces the system through the relocation mode,
so all transitions in $\Sigma$ are of the form $q \rightarrow q^{\recon}$ or
$q^{\recon} \rightarrow q$. The
guards are then
$G(q^{\bak} \rightarrow q^{\recon})  := \mathcal{B}^{\star} \times \{0\}, \; G(q^{\star} \rightarrow q^{\recon})  := \mathcal{B}^{\star} \times \{0\}, 
G(q^{\recon} \rightarrow q^{\star}) := \mathcal{B}^{\star} \times \{\delta\}, \; G(q^{\recon} \rightarrow q^{\bak})   := (\mathcal{B}\setminus\mathcal{B}^{\star}) \times \{\delta\},$
and the reset map simply resets the timer on every relocation,
$R(q^{\recon} \rightarrow q_2, (\vec{\xi},s)) = (\vec{\xi}, 0)$, leaving the physical
state unchanged.

Informally, the primary controller may be engaged from the backup mode whenever the
state lies in $\mathcal{B}^{\star}$. On the other hand, the backup controller is triggered from the primary controller as follows: let $\mathcal{S}^{\mathcal{B}^{\star} \rightarrow \mathcal{B}^{\star}} := \{\vec{\xi} \in \mathcal{B}^{\star} \mid \phi(\delta, \vec{\xi}) \in \mathcal{B}^{\star}\}, 
\mathcal{S}^{\mathcal{B}^{\star} \rightarrow \mathcal{B}} := \{\vec{\xi} \in \mathcal{B}^{\star} \mid \phi(\delta, \vec{\xi}) \in \mathcal{B} \setminus \mathcal{B}^{\star}\}$, 
where $\mathcal{B}^{\star} = \mathcal{S}^{\mathcal{B}^{\star} \rightarrow \mathcal{B}^{\star}} \cup \mathcal{S}^{\mathcal{B}^{\star} \rightarrow \mathcal{B}}$. When the primary controller requires an allocation change from a state in $\mathcal{S}^{\mathcal{B}^{\star} \rightarrow \mathcal{B}^{\star}}$, then the hybrid architecture allows the mode change $q^{\star} \rightarrow q^{\recon}   \rightarrow q^{\star}$, since after relocation, the system state will again be in $\mathcal{B}^{\star}$. On the other hand, if a relocation is required from $\mathcal{S}^{\mathcal{B}^{\star} \rightarrow \mathcal{B}}$, then the backup mode is reached after relocation as $q^{\star} \rightarrow q^{\recon}   \rightarrow q^{\bak}$.

\noindent
\textbf{Initial conditions}
Finally, assuming the system starts within the safe set with the chasers already
in place, we take
$$
\text{Init} := \bigl(\{(\vec{\xi}, q^{\bak}) \mid \vec{\xi} \in \mathcal{B}\}
\cup \{(\vec{\xi}, q^{\star}) \mid \vec{\xi} \in \mathcal{B}^{\star}\}\bigr) \times \{0\},
$$
which is consistent with the invariants above. 

We conclude by proving that the resulting hybrid system architecture ensures safety for the system during tracking.

\begin{figure*}[t!]
    \centering
    \includegraphics[width=0.9\linewidth]{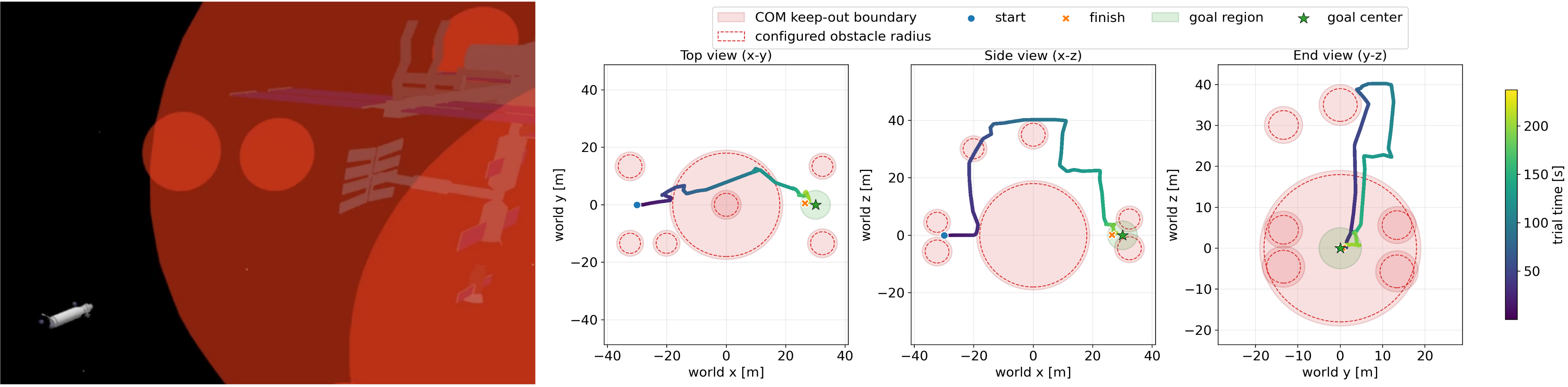}
    \caption{Transportation scenario and trajectory of the target center of mass under the operator's commands. Video available at: \url{https://youtu.be/W_9vIz7-wOk}
    }
    \label{fig:whole trajectory}
\end{figure*}

\begin{proposition}\label{prop:last prop}
    Let the hybrid dynamical system $H$ as per \eqref{eq:hybrid system}, with initial condition $(\vec{\xi}(0),q(0),s(0))\in\mathrm{Init}$. Then, the system trajectory satisfies $\vec{\xi}(t) \in \mathcal{B},\; \forall t\geq0$. 
\end{proposition}
\begin{proof} We argue invariance in the set $\mathcal{B}$ (and thus safety) mode by mode.\\

\textbf{Backup mode:} On $\mathcal{B}$ the controller \eqref{eq:backup controller} is always feasible. Indeed, by Assumptions~\ref{ass:controllability}--\ref{ass:min chasers}, for every state $\vec{\xi}$, the fixed allocation $M^{\bak}$ admits an input $\bar{\vec{u}} \in \mathbb{U}$
with $\bar{\vec{w}} := \vec{w}(M^{\bak},\bar{\vec{u}}) =
[\bar{\vec{f}}^\top, \bar{\vec{\tau}}^\top]^\top$ such that
$\frac{\vec{d}^\top R(\vec{q})}{\rho m}\bar{\vec{f}} := \bar{a} \geq a_{\max}(M^{\bak}) > 0$, with
$a_{\max}(M^{\bak})$ as in \eqref{eq:min-max-acc}. Since for every
$\vec{\xi} \in \mathcal{B}$, we have
$
\frac{\vec{v}^\top\vec{d}}{\rho\sqrt{kh + \varepsilon}} \geq -1$, $\kappa \, b(\vec{\xi}) \geq 0$, and $\frac{\|\vec{v}\|^2}{\rho} - \frac{(\vec{d}^\top\vec{v})^2}{\rho^3}
\geq \frac{\|\vec{v}\|^2}{\rho} - \frac{\|\vec{v}\|^2\|\vec{d}\|^2}{\rho^3} = 0,
$
then
$
Q(\vec{\xi},\bar{\vec{w}})
= \underbrace{\frac{\|\vec{v}\|^2}{\rho} - \frac{(\vec{d}^\top\vec{v})^2}{\rho^3}}_{\geq \, 0}
+ \underbrace{\frac{k \, \vec{d}^\top\vec{v}}{2\rho\sqrt{kh + \varepsilon}}}_{\geq \, -k/2 \, = \, -a_{\max}(M^{\bak})}
+ \underbrace{\frac{\vec{d}^\top R(\vec{q}) \bar{\vec{f}}}{\rho m}}_{= \, \bar{a}}
+ \underbrace{\kappa \, b(\vec{\xi})}_{\geq \, 0} \\
\geq \bar{a} - a_{\max}(M^{\bak}) \geq 0.
$
So $\bar{\vec{w}}$ is a feasible solution of \eqref{eq:backup controller} and invariance of $\mathcal{B}$
follows from the satisfaction of the CBF constraint and Nagumo's theorem \cite{ames2019control}. 

\textbf{Primary mode:} By Prop. \ref{prop:main propsition}, $b^{\star}$ is
differentiable on $\mathcal{B}^{\star}_{+}$, so $Q^{\star} \geq 0$ in
\eqref{eq:cool controller} is a well-posed CBF constraint and, whenever \eqref{eq:cool controller} is feasible, renders $\mathcal{B}^{\star} \subseteq
\mathcal{B}$ forward invariant. If for a state $\vec{\xi} \in \mathcal{B}^{\star}$ the controller \eqref{eq:cool controller} is infeasible, a relocation is triggered, bringing mode $q^{\recon}$ and then to either $q^{\bak}$ or $q^{\star}$. 

\textbf{Relocation mode:} Here $\vec{w} = \vec{0}$ and the state follows the
unforced flow for $\delta$ time units. Both incoming guards lie in
$\mathcal{B}^{\star}$ and the reset leaves $\vec{\xi}$ unchanged, so entry occurs
at some $\vec{\xi} \in \mathcal{B}^{\star}$ and hence $\phi(\tau,\vec{\xi}) \in
\mathcal{B}$ for all $\tau \in [0,\delta]$ by definition of the predictive safe
set. The target is thus safe throughout the unactuated window, though it
may leave $\mathcal{B}^{\star}$.

\textbf{Concatenation:} Each visit to $q^{\recon}$ lasts $\delta > 0$ and is
separated from the next by an interval in a controlled mode. Starting from $\text{Init}$ and concatenating the three cases gives
$\vec{\xi}(t) \in \mathcal{B}$ for all $t \geq 0$. Since $q^{\bak}$ is invariant
by the first step, it acts as a safe terminal fallback.
\end{proof}
\subsection{Lazy Updates and Convex Relaxation}\label{sec:relaxation}
The controller in \eqref{eq:cool controller} requires solving a MIQP at every time step to obtain the optimal force 
vector $\vec{u}^{\star}$ and allocation $\alloc^{\star}$, which can be 
impractical for real-time hardware implementation. To improve tractability, we introduce two heuristics: \textit{lazy 
updates} and 
\textit{convex relaxation}.

\paragraph{Lazy updates} Suppose that at time $t$ we solve 
\eqref{eq:cool controller} and obtain the optimal allocation 
$M^{\star}(\vec{\xi}(t))$. Rather than resolving the full MIQP at every 
step, we hold the allocation fixed over a window of length $\delta_{\text{lazy}}$, i.e., 
$M^{\star}(\vec{\xi}(\tau)) = M^{\star}(\vec{\xi}(t))$ for all 
$\tau \in [t, t+\delta_{\text{lazy}}]$. Over this interval, 
\eqref{eq:cool controller} reduces to a QP in $\vec{u}$ alone, which can be 
solved at a much faster rate than the full MIQP. The allocation is then 
refreshed only at the sampling instants $t^{\text{sol}} = N\delta_{\text{lazy}}$, 
$N \in \mathbb{N}$.

\paragraph{Convex relaxation} At each refresh instant $t^{\text{sol}}$ we 
still need an (approximately) optimal integer allocation. Instead of solving 
the MIQP directly, we approximate it with two consecutive QPs.

\textit{Step 1 (relaxed allocation).} We relax the binary constraint 
$M \in \mathbb{M}$ to the continuous set 
$M_{\text{rlx}} \in \mathbb{M}_{\text{rlx}} := \{M\in [0,1]^{n_{\th}\times n_a} \mid \sum_{i=1}^{n_{\th}} \alloc_{ij} = 1,\,  \sum_{j=1}^{n_{a}} \alloc_{ij} \leq 1\}$ and solve
$$
\begin{aligned}
\vec{w}_{\text{rlx},1}^{\star}(\vec{\xi}(t), q^{\star}) &:= \hspace{-0.3cm} 
\underset{\vec{u} \in \mathbb{U},\; \alloc \in \mathbb{M}_{\text{rlx}},\; \mu \in \mathbb{R}_{\geq 0}}{\text{argmin}} 
\hspace{-0.3cm}\|\vec{u}\|_{2}^2 + \lambda_{\mu}\mu, \; \text{s.t.:}\; \eqref{eq:primary controller:constraints}
\end{aligned}
$$
This yields a relaxed allocation $M^{\star}_{\text{rlx}}$, which we round 
entry-wise to obtain a feasible integer allocation 
$M^{\star}_{\text{round}} \in \mathbb{M}$.

\textit{Step 2 (force recovery).} Holding $M^{\star}_{\text{round}}$ fixed, 
we solve a second QP over $\vec{u}$ alone:
$$
\begin{aligned}
&\vec{w}_{\text{rlx},2}^{\star}(\vec{\xi}(t), q^{\star}) := 
\underset{\vec{u} \in \mathbb{U},\; \mu \in \mathbb{R}_{\geq 0}}{\text{argmin}}  
\, \|\vec{u}\|_{2}^2 + \lambda_{\mu}\mu\; \text{subject to:} \\
&L(\vec{\xi}(t), \vec{w}(M^{\star}_{\text{round}}, \vec{u})) \leq \mu, 
Q^{\star}(\vec{\xi}(t), \vec{w}(M^{\star}_{\text{round}}, \vec{u})) \geq 0.
\end{aligned}
$$
From which we obtain a wrench satisfying the safety and stability constraints. 

These heuristics do not affect the safety of the system, as only the primary controller is approximated in such a manner, while the backup configuration can always be reached safely and computed in real-time.   
\section{Simulations Results}\label{sec:sim}
To validate the approach, we select the scenario of transporting an expendable rocket (the target) from one side of the International Space Station (ISS) to the other side of the station with the minimum number of chasers, $ n_a=4$. We use Gazebo as the physics engine for the simulations, with target and chaser mesh files from the NASA catalog\footnote{https://science.nasa.gov/3d-resources/}. 

We select 40 actuation locations on the target, where chasers can freely exert push force.   The target and chasers parameters, as well as the controller parameters used, are shown in Table \ref{tab:parameters}. We purposely scale down size, mass, and inertia compared to a real expendable rocket body for this demonstration, to obtain faster dynamics that are more challenging to compensate for, compared to a real space mission where operations are limited to small accelerations. We use a remote controller to provide reference twist commands $\vec{\eta}_{d} = [v_x,v_y,v_x,\omega_x,\omega_y,\omega_z]$ with linear velocity commands expressed in the inertial frame $\mathcal{I}$ and angular velocity commands expressed in the body frame of the target.

In Figure \ref{fig:whole trajectory}, we show the full target trajectory over time as it moves from the starting pose to the goal area marked as a green sphere. The keep-out zone of the ISS and a few extra hazardous areas are marked in red. Within the controller, we always consider the closest obstacle to the target for computing the CBF $b^{\star}(\vec{\xi}(t))$ and is gradient. Figure \eqref{fig:controller trajectory} shows the value of the CLF and CBFs during the experiments, as well as the mode switching of the controller and the provided input commands from the operator. At approximately 20 seconds, the controller successfully hits the backup mode to avoid a collision with the ISS, where the $b(\vec{\xi})$ hits the minimum value of $0.03$. At the same time the last two panels in \eqref{fig:controller trajectory} show the computational time for the hybrid controller (accounting for solving the optimization in \eqref{eq:backup controller} and \eqref{eq:cool controller}, checking obstacle proximity, recording the operator input, evaluating, \(V(\vec{\xi}(t))\), \(b(\vec{\xi}(t))\), \(b^*(\vec{\xi}(t))\), \(\tau^*(\vec{\xi}(t))\), and transition guards). When a reallocation change is required, the relaxed QP median solve time is 27 times faster than the MIQP solver. In real deployment, where the dynamics of the systems are less excited, the solution time of the MIQP could still yield reasonable real-time performance with an absolute computational time not exceeding 0.1s.
\begin{table}
\caption{Parameters of the reported trial.}
\label{tab:parameters}
\centering
\setlength{\tabcolsep}{4pt}
\renewcommand{\arraystretch}{1.1}
\begin{tabular}{@{}ll|ll|ll@{}}
\toprule
\toprule
$n_a$                     & 4        & $\delta$          & 3.0\,s & $m$   & 20.0\,kg \\
$u_{\max}$                & 10\,N    & $K$               & 2.0    & $\mathrm{diag}(J)$ & $(0.57, 3.84, 4.19)$\,kg\,m$^2$\\
$n_a \cdot u_{\max}$      & 40\,N    & $\kappa$          & 1.0    & $\bar{\eta}$& 1\\
$n_{\th}$                 & 40       & $\lambda_{\mu}$   & 2000   & $k_{f}$& 10\\
                          &          & $\delta_{\text{lazy}}$   & 1s     &  &\\
\bottomrule
\end{tabular}
\end{table}
\begin{figure}
    \centering
    \includegraphics[width=\linewidth]{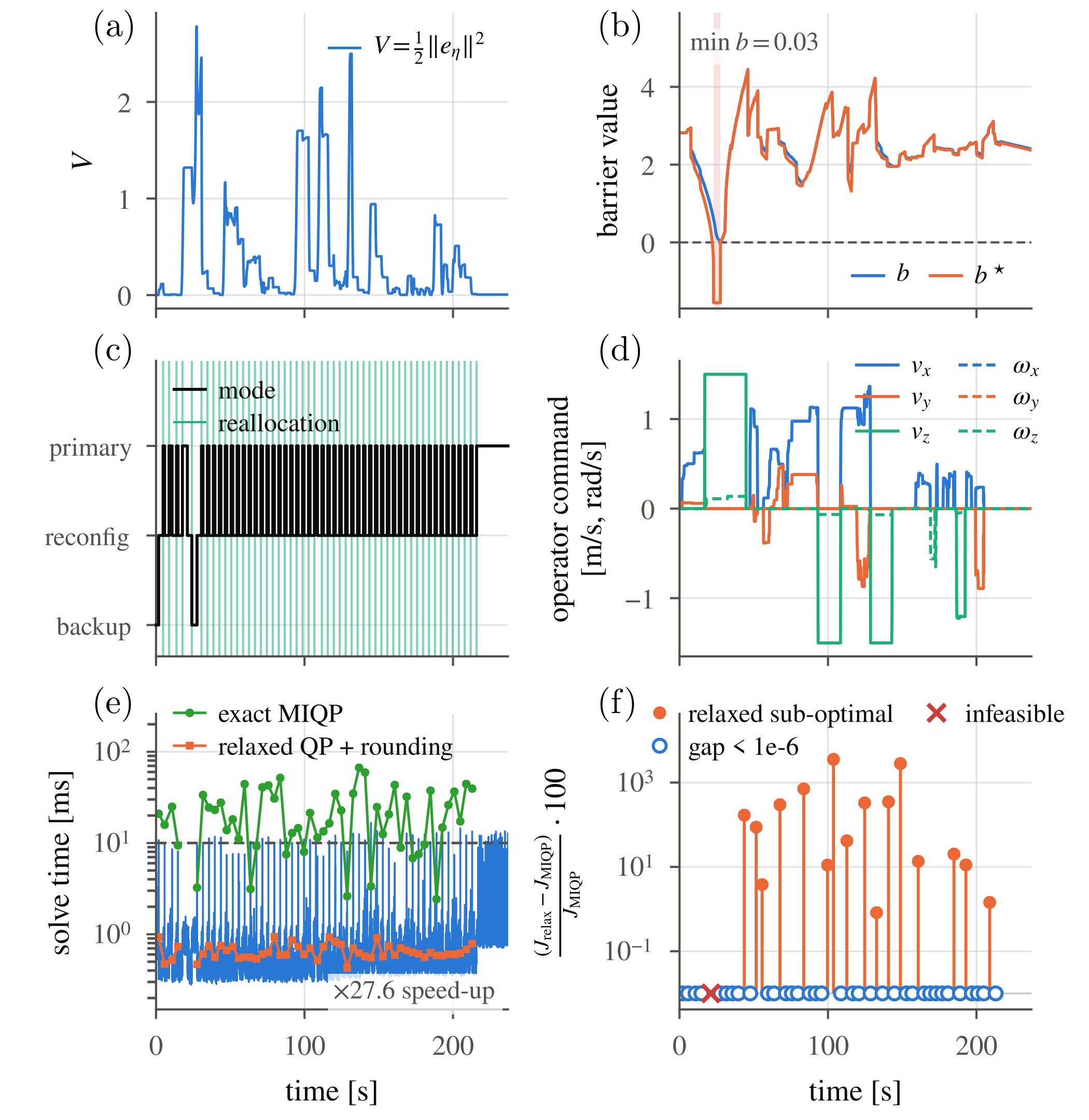}
    \caption{Simulation results. (a) Lyapunov function $V(\vec{\xi}(t))$, (b) Barrier function value for $b(\vec{\xi}(t))$ and  $b^{\star}(\vec{\xi}(t))$, (c) Modes switches within the hybrid system $H$, (d) Operator velocity commands, (e) controller solver-time and comparative study MIQP vs relaxed solution solver time (f) Relative optimality gap between MIQP solver and relaxed-QP solver.}
    \label{fig:controller trajectory}
\end{figure}

\section{Conclusion}\label{sec:conclusion}
We presented an approach for transporting a passive target while preserving 
its safety, using a set of chaser spacecraft as a virtual actuation system 
subject to reconfiguration switches. In future work, we aim to extend the 
framework to consider realistic contact forces between the chasers and 
the target and uncertainties in the target's mass and inertia. 

\section{Acknowledgements}
The authors used generative AI tools to help develop the simulation code and software architecture. All AI-assisted content was reviewed, verified, and edited by the authors, who take full responsibility for the accuracy and originality of the manuscript.

\appendix
\subsection{Proof of Proposition \ref{prop:main propsition} }\label{appendix:proof result}

\begin{proof} We prove that $T^{\star}$ is a singleton and, by Daskins' theorem \cite{bertsekas1997nonlinear}, that differentiability of $b^{\star}(\vec{\xi})$ holds on $\mathcal{B}_{+}^\star$.\par

\textbf{$\bm{T^{\star}}$ is a singleton}: $b^{\star}(\vec{\xi})$ is the minimum of the functon 
\[
g_{\vec{\xi}}(\tau)
:=
b(\phi(\tau,\vec{\xi})),
\qquad
\tau\in[0,\delta].
\]
Since $g_{\vec{\xi}}$ is continuous and $[0,\delta]$ is compact, at least one minimizer exists $\tau^{\star}_{\vec{\xi}}:= \tau^{\star}(\vec{\xi})$. Namely, $\tau^{\star}_{\vec{\xi}}$ being a solution requires the existence of a pair of Lagrangian multipliers $\lambda^{\star}_{1,\vec{\xi}}:= \lambda^{\star}_{1}(\vec{\xi}), \lambda^{\star}_{2,\vec{\xi}}:= \lambda^{\star}_{2}(\vec{\xi})$, where 
$w_{\vec{\xi}}^{*}:= (\tau^{\star}_{\vec{\xi}}, \lambda_{1,\vec{\xi}}^{\star}, \lambda_{2,\vec{\xi}}^{\star})$, that satisfy the following KKT system
\begin{subequations}\label{eq:kkt}
\begin{align}
&\diff{\tau} g(\tau^{\star}_{\vec{\xi}}) - \lambda_{1,\vec{\xi}}^{\star} +\lambda_{2,\vec{\xi}}^{\star} = 0, && \hspace{-0.3cm}\text{(stationarity)} \label{eq:kkt_stationarity} \\
&\lambda_{1,\vec{\xi}}^{\star} \tau^{\star}_{\vec{\xi}} = 0 , \; \lambda_{2,\vec{\xi}}^{\star} (\tau^{\star}_{\vec{\xi}} - \delta) = 0, &&  \hspace{-0.3cm}\text{(complementarity)}\\
&0\leq \tau^{\star}_{\vec{\xi}} \leq \delta, \; \lambda_{1,\vec{\xi}}^{\star},\lambda_{2,\vec{\xi}}^{\star} \geq 0 .&&  \hspace{-0.3cm} \text{(feasibility)}
\end{align}
\end{subequations}
which, by inspection, admits only three solution cases
$$
\begin{array}{llll}
\text{S1:}&\tau^{\star}_{\vec{\xi}} = 0 & \diff{\tau} g(\tau^{\star}_{\vec{\xi}}) = \lambda_{1,\vec{\xi}}^{\star}  & \lambda_{1,\vec{\xi}}^{\star} \geq 0, \\
\text{S2:}&\tau^{\star}_{\vec{\xi}} = \delta & \diff{\tau} g(\tau^{\star}_{\vec{\xi}}) = -\lambda_2 &  \lambda_{2,\vec{\xi}}^{\star}  \geq 0, \\
\text{S3:}&\tau^{\star}_{\vec{\xi}} \in (0,\delta) & \diff{\tau} g(\tau^{\star}_{\vec{\xi}}) = 0 & \lambda_{2,\vec{\xi}}^{\star}, \lambda_{1,\vec{\xi}}^{\star}  =0. \\
\end{array}
$$
We now note that any triplet $w_{\vec{\xi}}^{\vec{\star}}$ satisfying \eqref{eq:kkt}, also satisfies the Linear Independence Constraint Qualification (LICQ) condition \cite[Def. 2.3]{pacaud2025sensitivity} since $\tau^{\star}_{\vec{\xi}} = 0$ and $\tau^{\star}_{\vec{\xi}} - \delta=0$ can not be active at the same time. Moreover, letting the \textit{critical cone} \cite[Eq. 2.25]{pacaud2025sensitivity}
$$
\mathcal{D}(w_{\xi}^{\star}) := \left\{ s \in \mathbb{R} \Bigl|
\begin{array}{ll}
    s  = 0   & \text{if} \;\; \lambda_{1,\vec{\xi}}^{\star} > 0 ;\; \tau^{\star}_{\vec{\xi}} = 0 \\
    s  = 0  & \text{if} \; \;\lambda_{2,\vec{\xi}}^{\star} > 0 ;\; \tau^{\star}_{\vec{\xi}} = \delta \\
    \mathbb{R} & \text{else}
\end{array}
\right\}
$$
then $w_{\vec{\xi}}^{\star}$ respects the Strong Second-order Sufficiency Condition (SSOSC) \cite[Def. 2.13]{pacaud2025sensitivity} if, 
$
s^2 \cdot \difftwo{\tau} g(\tau^{\star}_{\vec{\xi}}) > 0, \; \forall s\in \mathcal{D}(w_{\xi}^{\star})\setminus \{0\},
$
which simplifies to
\begin{equation}\label{eq:ssocs cases}
\difftwo{\tau} g(\tau^{\star}_{\vec{\xi}}) > 0 \; \text{if} \;\begin{cases}
     \makebox[0pt][l]{(i)}\phantom{(iii)}   \lambda_{2,\vec{\xi}}^{\star} = 0, \; \tau^{\star}_{\vec{\xi}} = \delta \;  \\
      \makebox[0pt][l]{(ii)}\phantom{(iii)}  \lambda_{1,\vec{\xi}}^{\star} = 0,\; \tau^{\star}_{\vec{\xi}} = 0\;\\
      \makebox[0pt][l]{(iii)}\phantom{(iii)} \lambda_{1,\vec{\xi}}^{\star} = 0,\; \lambda_{2,\vec{\xi}}^{\star}= 0.
\end{cases}
\end{equation}

If SSOSC property in \eqref{eq:ssocs cases} holds (which we prove next), then the triplet $w_{\vec{\xi}}^{\star}$ satisfying \eqref{eq:kkt} is unique \cite[Prop. 3.6]{pacaud2025sensitivity} and $T^{\star}$ is a singleton. We thus derive next the second derivative $\difftwo{\tau} g(\tau^{\star}_{\vec{\xi}})$. In particular, recall that the function $b(\vec{\xi})$ in \eqref{eq:barrier functions} depends only on the position $\vec{p}$ and velocity $\vec{v}$, and for unforced solution $\phi(\tau, \vec{\xi})$, these evolve as
\begin{equation}
    \vec{p}(\tau) = \vec{p}(0) + \vec{v}\tau, \quad  \vec{v}(\tau) = \vec{v}(0), \; \forall \tau \in [0,\delta].
\end{equation}
Let the new convenient coordinate variables
$$
\alpha(\tau) = \frac{\vec{d}(\tau)^\top\vec{v}(\tau)}{\rho(\tau)^2}, \;\; \beta(\tau) = \frac{\|\vec{v}(\tau)\|^2}{\rho(\tau)^2}, \;   c(\tau) = kh(\tau) + \epsilon,
$$
where $\diff{\tau} \rho(\tau) = \alpha(\tau) \rho(\tau), \diff{\tau} c(\tau) = k \alpha(\tau) \rho(\tau), \diff{\tau} \alpha (\tau)= \beta(\tau) - 2\alpha^2(\tau),  \diff{\tau} \beta(\tau) = -2\alpha(\tau)\beta(\tau)$. 
In these coordinates, $g_{\vec{\xi}}(\tau)$ and its derivatives are 
\begin{subequations}
\begin{align}
g_{\vec{\xi}}(\tau) &= \rho \left[\alpha  + \frac{\sqrt{c}}{\rho}\right], \label{eq:barrier value new coordinate} \\
\diff{\tau}g_{\vec{\xi}}(\tau) &= \rho \left[(\beta - \alpha^2) + \frac{k \alpha }{2\sqrt{c}}\right],\label{eq:barrier value new coordinate derivative} \\
\difftwo{\tau} g_{\vec{\xi}}(\tau) &= \rho\left[(\beta-\alpha^2)\left(\frac{k}{2\sqrt{c}}-3\alpha\right) - \frac{k^2\alpha^2\rho}{4c^{3/2}}\right], \label{eq:barrier value new coordinate second derivative}
\end{align}
\end{subequations}
where $\tau$ dependency is omitted for brevity. Moreover, for all $\vec{\xi} \in \mathcal{B}_{+}^\star \subset \mathcal{B}$ the following inequalities hold for all $\tau \in [0,\delta]$:
\begin{subequations}\label{eq:properties needed}
\begin{align}
     &\beta(\tau) - \alpha(\tau)^2 \geq 0, \quad (\text{Cauchy–Schwarz}) \label{eq:cond1}\\
     &\alpha(\tau) \rho(\tau) + \sqrt{c(\tau)} - \sqrt{\epsilon} \geq 0, \quad \text{(\eqref{eq:barrier value new coordinate} and $\vec{\xi} \in \mathcal{B}^\star$)}\\  
    &c(\tau) > \epsilon > 0, \; \beta(\tau)> 0, \; \rho(\tau) > r >0. 
\end{align}
\end{subequations}
Hence, to prove SSOCS we prove that the three mutually exclusive cases (i)-(iii) in \eqref{eq:ssocs cases} hold. The proof is identical for all cases. Namely, for all cases (1)-(3) in \eqref{eq:ssocs cases} the solution $w^\star_{\xi}$ should satisfy KKT system \eqref{eq:kkt}, which yields $\diff{\tau}g_{\vec{\xi}}(\tau^{\star}_{\vec{\xi}}) = 0$. Letting for convenience $\beta^{\star} := \beta(\tau_{\vec{\xi}}^{\star})$, $c^{\star} := c(\tau_{\vec{\xi}}^{\star})$, and $\alpha^{\star} := \alpha(\tau_{\vec{\xi}}^{\star})$, then
$
\diff{\tau}g_{\vec{\xi}}(\tau^{\star}_{\vec{\xi}}) = 0 \Rightarrow (\beta^{\star} - \alpha^{\star 2}) = - \frac{k \alpha^{\star} }{2\sqrt{c^{\star}}}.
$
Since $\beta^\star-\alpha^{\star2}\geq0$ and $c^\star>0$, this implies $\alpha^\star\leq0$. At the same time, selecting $\alpha^* = 0$ yields $\diff{\tau}g_{\vec{\xi}}(\tau^{\star}_{\vec{\xi}}) = \rho \beta^{\star}>0$, which contradicts $\diff{\tau} g_{\vec{\xi}}(\tau^{\star}_{\vec{\xi}}) = 0$. Thus we conclude that $\alpha^* <0$. Replacing $(\beta^{\star} - \alpha^{\star 2}) = - \frac{k \alpha^{\star} }{2\sqrt{c^{\star}}}$ in \eqref{eq:barrier value new coordinate second derivative} we have 
\begin{equation}\label{eq:second derivative}
\difftwo{\tau} g_{\vec{\xi}}(\tau^{\star}_{\vec{\xi}}) = \frac{\alpha^{\star} k \rho}{4 c^{\star\; 3/2}} \cdot (6 \alpha^{\star} c^{\star} - k (\alpha^{\star} \rho^{\star} + \sqrt{c^{\star}})).
\end{equation}
recalling that $(\alpha^{\star} \rho^{\star} + \sqrt{c^{\star}}) > \sqrt{\epsilon} >0$, then replacing $\alpha^* <0$ yields $\difftwo{\tau} g_{\vec{\xi}}(\tau^{\star}_{\vec{\xi}})>0$, proving the SSOCS at the solution. 

\textbf{Differentiability} By the fact that $T^{\star} = \{\tau^\star\}$ is a singleton, and that SSOCS and LIQC hold, Danskin's Theorem \cite[Thm 4.6]{pacaud2025sensitivity} ensures the differentiability of $b^{\star}(\vec{\xi})$ with 
$$
\diff{\vec{\xi}} b^{\star}(\vec{\xi}) = \diff{\vec{\xi}}b (\phi(\tau_{\vec{\xi}}^{\star},\vec{\xi}) = \diff{\vec{\xi}}b \diff{\vec{\xi}} \phi(\tau_{\vec{\xi}}^{\star},\vec{\xi}) = \diff{\vec{\xi}}b \;  \Phi(\tau_{\vec{\xi}}^{\star},\vec{\xi})
$$
where $\diff{\vec{\xi}} \phi(\tau_{\vec{\xi}}^{\star},\vec{\xi}) := \Phi(\tau_{\vec{\xi}}^{\star},\vec{\xi})$ is the sensitivity of the unforced solution with respect to the initial condition $\vec{\xi}$ evaluated at time $\tau^\star$ as defined in \eqref{eq:state transition matrix} \cite[Ch. 3.3]{khalil2002nonlinear}. \hfill 
\end{proof}

\addtolength{\textheight}{-12cm}   





\bibliographystyle{ieeetr}
\bibliography{biblio.bib}

\end{document}